\documentclass[10pt,twocolumn,letterpaper]{article}

\usepackage{iccv}
\usepackage{times}
\usepackage{epsfig}
\usepackage{graphicx,booktabs}
\usepackage{amsmath}
\usepackage{amssymb}
\usepackage{subfigure}
\usepackage{multirow}
\usepackage{multicol}
\usepackage{makecell}
\usepackage{todonotes}
\usepackage{amsthm}
\usepackage{authblk}

\newtheorem{lemma}{Lemma}
\newtheorem{definition}{Definition}
\newtheorem{proposition}{Proposition}
\usepackage[breaklinks=true,bookmarks=false]{hyperref}

\iccvfinalcopy 

\def\iccvPaperID{12729} 

\ificcvfinal\fi

\begin{document}

\title{Deep Directly-Trained Spiking Neural Networks for Object Detection}

\author[1]{Qiaoyi Su}
\author[2,7]{Yuhong Chou}
\author[3]{Yifan Hu}
\author[4]{Jianing Li}
\author[5,7]{Shijie Mei}
\author[6]{Ziyang Zhang}
\author[7]{Guoqi Li}
\affil[1]{School of Artificial Intelligence, University of Chinese Academy of Sciences}
\affil[2]{College of Artificial Intelligence, Xi’an Jiaotong University}
\affil[3]{Department of Precision Instrument, Tsinghua University}
\affil[4]{School of Computer Science, Peking University}
\affil[5]{School of Vehicle and Mobility, Tsinghua University}
\affil[6]{Advanced Computing and Storage Lab, Huawei Technologies Co Ltd.}
\affil[7]{Institute of Automation, Chinese Academy of Sciences}
\renewcommand*{\Affilfont}{\small\it}
\maketitle
\ificcvfinal\thispagestyle{empty}\fi

\begin{abstract}
   Spiking neural networks (SNNs) are brain-inspired energy-efficient models that encode information in spatiotemporal dynamics. Recently, deep SNNs trained directly have shown great success in achieving high performance on classification tasks with very few time steps. However, how to design a directly-trained SNN for the regression task of object detection still remains a challenging problem. To address this problem, we propose EMS-YOLO, a novel directly-trained SNN framework for object detection, which is the first trial to train a deep SNN with surrogate gradients for object detection rather than ANN-SNN conversion strategies. Specifically, we design a full-spike residual block, EMS-ResNet, which can effectively extend the depth of the directly-trained SNN with low power consumption. Furthermore, we theoretically analyze and prove the EMS-ResNet could avoid gradient vanishing or exploding. The results demonstrate that our approach outperforms the state-of-the-art ANN-SNN conversion methods (at least 500 time steps) in extremely fewer time steps (only 4 time steps). It is shown that our model could achieve comparable performance to the ANN with the same architecture while consuming 5.83$\times$ less energy on the frame-based COCO Dataset and the event-based Gen1 Dataset. Our code is available in \url{https://github.com/BICLab/EMS-YOLO}.
\end{abstract}

\section{Introduction}

Object detection is a key and challenging problem in computer vision. It aims to recognize multiple overlapped objects and locate them with precise bounding boxes. This task has many applications in various fields, such as autonomous driving~\cite{balasubramaniam2022object}, security surveillance~\cite{zhang2019object}, and medical imaging~\cite{litjens2017survey}. Most existing frameworks (\eg, YOLO series~\cite{redmon2016you}, RCNN series~\cite{girshick2014rich}) for object detection use artificial neural networks (ANNs), which have high performance but also high computational complexity and energy consumption. Spiking neural networks (SNNs), known as the third generation of neural networks~\cite{maass1997networks,roy2019towards}, potentially serves as a more efficient and biologically inspired way to perform object detection. Specifically, SNNs utilize binary signals (spikes) instead of continuous signals for neuron communication, which reduces data transfer and storage overhead. Furthermore, the SNNs exhibit asynchronous computation and event-driven communication, which could avoid unnecessary computation and synchronization overhead. When deployed on neuromorphic hardware~\cite{merolla2014million,poon2011neuromorphic}, SNNs show great energy efficiency. 


However, most SNNs for object detection are converted from ANNs, which have some limitations. For example, Spiking-Yolo~\cite{kim2020spiking,kim2020towards} needs at least 3500 time steps to match the performance of the original ANN. Spike Calibration~\cite{li2022spike} can reduce the time steps to hundreds, but it still depends on the performance of the original ANN model. Moreover,  most methods for converted SNNs are only applicable for static images, and not suitable for sparse event datas, because their dynamics are designed to approximate the expected activation of the ANN and fail to capture the spatiotemporal information of DVS data~\cite{deng2020rethinking}. A promising approach is to train SNNs directly with surrogate gradient, which can achieve high performance with few time steps and process both static images and event data efficiently.

Another challenge is to deal with the multi-scale object features in object detection, which demands the network has sufficient representation capability. Correspondingly, deep structure training is needed. Existing models on object detection are limited to shallow structures~\cite{cordone2022object} or hybrid structures~\cite{lien2022sparse,johansson2021training} that may not be deployed on some neuromorphic hardware~\cite{davies2018loihi,akopyan2015truenorth,liu2019live} where only spike operation is allowed. To achieve deep direct training of SNNs, Hu \etal.~\cite{hu2021advancing} and Fang \etal .~\cite{fang2021deep} explored on classification tasks and proposed MS-ResNet and SEW-ResNet respectively to overcome the vanishing/exploding gradient problems and advanced the directly-trained SNNs to depths greater than 100 layers. Unfortunately, multi-scale transformation of channels and dimensions when extracting features of objects of different sizes will cause the problem of increased energy consumption due to non-spike convolutional operations in their networks to be prominent in the object detection task. Therefore, the problem of high energy consumption caused by the non-spike convolutional operations urgently needs to be addressed.


To tackle these problems, we propose a novel directly trained SNN for object detection based on the YOLO framework (EMS-YOLO). Our model is the first to use surrogate gradients to train a deep and large-scale SNN for object detection without converting from ANNs. Specially, to deal with the multi-scale object features in object detection, we design a new full-spike energy-efficient residual block, EMS-ResNet, that avoids redundant MAC operations caused by non-spike convolution. Our model can achieve high performance with few time steps and handle both static images and event data efficiently. Compared with other converted or hybrid SNNs for object detection, our model has higher performance and lower energy consumption. 

Our major contributions of this paper can be summarized as follows:
\begin{itemize}
    \item  We propose EMS-YOLO, a novel directly trained spiking neural network for object detection, which could achieve better performance than the advanced ANN-SNN conversion methods while requiring only 4 time steps and inferencing in real time. 
    \item We design an \textbf{E}nergy-efficient \textbf{M}embrane-\textbf{S}hortcut ResNet, EMS-ResNet, that enables full spiking in the network thus reducing power consumption. Moreover, we theoretically analyze that it can be trained deeply since it avoids gradient disappearance or explosion.
    \item The experiments on COCO and the Gen1 Datasets demonstrate that our models could achieve comparable performance to the ANN with the same architecture meanwhile reducing $5.83\times$ energy consumption. 
\end{itemize}


\section{Related Work}
\subsection{Deep Spiking Neural Networks}
The training strategies for deep SNNs are mainly divided into ANN-SNN conversion and directly training SNNs. The essence of ANN-SNN conversion is to approximate the average firing rate of SNNs to the continuous activation value of ANNs that use ReLU as the nonlinearity~\cite{cao2015spiking,diehl2015fast}. The performance of the converted SNN relies on the original ANN, and it is tough to get high-performance SNNs with low time delay and suffers performance losses during conversion. Moreover, the converted SNNs cannot work well on the sparse event data which can be effectively combined with neuromorphic hardware.

Inspired from the explosive development of ANNs, the researchers use the surrogate gradient to achieve directly training of SNNs. The mature back-propagation mechanism and diverse coding schemes~\cite{neftci2019surrogate,kaiser2020synaptic} have enabled directly-trained SNNs to work well on short time steps while requiring low power consumption. Zheng \etal.~\cite{zheng2021going} proposed the threshold-dependent batch normalization (TDBN) technique which could extend SNNs from a shallow structure (\textless 10 layers) to a deeper structure (50 layers) based on the framework of STBP~\cite{wu2018spatio}. Hu \etal.~\cite{hu2021advancing} and Fang \etal.~\cite{fang2021deep} advanced the achievement of high performance on classification tasks. Currently, there is some work that uses deep directly training SNN for regression tasks like object tracking~\cite{zhang2022spiking,xiang2022spiking}, image reconstruction~\cite{zhu2022event}, while the object detection task has not yet been explored.

\subsection{Energy-Efficient Object Detection}
 Significant vision sensors in object detection include frame-based and event-based cameras~\cite{li2022retinomorphic,posch2010qvga}, where the latter can handle challenging scenes such as motion blur, overexposure, and low light, and have become a hot spot in recent years. The mainstream deep learning-based detectors mainly fall into two categories: two-stage frameworks (RCNN series)~\cite{girshick2014rich}, and one-stage frameworks (YOLO series~\cite{redmon2016you}, SSD~\cite{liu2016ssd}, Transformer series~\cite{zhu2020deformable}). They are implemented based on ANNs, which achieve high performance but also bring high energy consumption. Therefore, some explorations of SNN-based object detection have attempted to provide more energy-efficient solutions.
 
 The earliest attempts~\cite{kim2020spiking,kim2020towards,li2022spike} were based on ANN-SNN conversion method which requires long inference time and cannot be applied to event camera data due to the inherent limitations of the approach. Some hybrid architectures~\cite{lien2022sparse,johansson2021training} tried to use directly trained SNN backbones and ANN detection heads for object detection, while these detection heads introduce an additional amount of parameters. In this work, we present the first attempt of object detection with fully and deep directly trained SNNs. 

\subsection{Spiking Residual Networks}
The ANN-SNN conversion does not concern the design of residual blocks oriented to the trainability of deep SNNs, while the residual block is mainly utilized to achieve loss-less accuracy~\cite{hu2021spiking,sengupta2019going}. For directly training of SNNs, the residual structure empowers it to train deeply. Zheng \etal.~\cite{zheng2021going} firstly obtain directly trained ResNet-34/50 with surrogate gradient. SEW-ResNet~\cite{fang2021deep} and MS-ResNet~\cite{hu2021advancing} extended the depth of the SNN model to over 100 layers. However, the former is essentially an integer multiplication operation of the residual network. The latter focuses mainly on the spiking of residual paths and ignores the non-spiking structure on shortcut paths. When applied to object detection tasks with varying dimensionality and number of channels, these non-spike convolutions can result in heavy energy consumption. Therefore, we design a full spiking residual network to exploit the energy efficiency of SNN.

\section{The Preliminaries of SNNs} 
\subsection{Spiking Neuron}
Neurons are the basic units of a neural network, which convert a barrage of synaptic inputs into meaningful action potential outputs. In ANNs, artificial neurons discard temporal dynamics and propagate information only in the spatial domain. In contrast, in SNNs, spiking neurons are more biologically plausible as they mimic the membrane potential dynamics and the spiking communication scheme. Typically, the Leaky Integrate-and-Fire (LIF) model~\cite{abbott1999lapicque}, the Hodgkin-Huxley (H-H) model~\cite{hodgkin1952quantitative} and the Izhikevich model~\cite{izhikevich2003simple} are the most famous ones. Among them, LIF models are widely adopted to construct SNNs due to its good trade-off between bio-plausibility and complexity. Also, they contain abundant biological features and consume less energy than the ANN neurons. In our work, we use the iterative LIF model proposed by Wu \etal.~\cite{wu2019direct} in the SNN model, which can be described as:
\begin{equation}
    V^{t+1,n+1}_i=\tau{V^{t,n+1}_i}(1-X^{t,n+1}_i)+\sum_j{W_{ij}^nX^{t+1,n}_j}
\end{equation}
\begin{equation}
    X^{t+1,n+1}_i=H(V^{t+1,n+1}_i-V_{th})
\end{equation}
where the $V^{t,n+1}_i$ is the membrane potential of the $i$-th neuron in the $n+1$ layer at the timestep $t$, $\tau$ is a decay factor for leakage. The input to a synapse is the sum of $j$ spikes $X^{t+1,n}_j$ with synaptic weights $W_{ij}^n$ from the previous layer $n$. $H(\cdot)$ denotes the Heaviside step function which satisfies $H(x)=1$ for $x\geq 0$, otherwise $H(x)=0$. As shown in the Figure~\ref{fig:ODSNN}, a firing activity is controlled by the threshold $V_{th}$, and the $V^{t+1,n+1}$ will be reset to $V_{rest}$ once the neuron emits a spike at time step $t+1$.

\subsection{Training Strategies}
In order to solve the problem that the spike cannot be differentiated in back-propagation, we use the surrogate gradient ~\cite{wu2018spatio} which can be represented as:
\begin{equation}
    \frac{\partial X_{i}^{t,n}}{\partial V_{i}^{t,n}}=\frac{1}{a} sign(|{V_{i}^{t,n}-V_{th}|\leq \frac{a}{2}})
\end{equation}
where $a$ is introduced to ensure the integral of the gradient is 1 and determines the curve steepness. 

We use the TDBN~\cite {zheng2021going} normalization method which considers spatial and temporal domains. The TDBN can be described as:
\begin{equation}
        V^{t+1,n+1}_i=\tau{V^{t,n+1}_i}(1-X^{t,n+1}_i)+{{\rm TDBN}(I_i^{t+1})}
\end{equation}
\begin{equation}
    {\rm TDBN}(I_i^{t+1})=\lambda_{i} \frac{\alpha V_{th}(I_i^{t+1}-\mu_{ci})}{\sqrt{\sigma_{ci}^{2}+\epsilon}}+\beta_{i}
\end{equation}
where $\mu_{ci},\sigma_{ci}^{2}$ are the mean and variation values for every channel using a mini-batch of sequential inputs $\{ I_{i}^{t+1}=\Sigma_{j}W_{ij}^{n}X_{j}^{t+1,n}|t=0,...,T-1\}$, $\epsilon$ is a tiny constant to avoid dividing by zero, $\lambda_i,\beta_i$ are two trainable parameters, and $\alpha$ is a threshold-dependent hyper-parameter.

\subsection{Energy Consumption}
The number of operations is often used to measure the computational energy consumption of neuromorphic hardware. In ANNs, each operation involves multiplication and addition (MAC) of floating point numbers, and Times of floating-point operations (FLOPs) are used to estimate computational burden. SNNs have the energy-efficient property in neuromorphic hardware since the neurons only participate in the accumulation calculation (AC) when they spike and could achieve the calculation with about the same number of synaptic operations (SyOPs). However, many current SNNs introduce additional MAC operations due to their design flaws. Thus, we quantify the energy consumption of vanilla SNNs as $E_{SNN}=\sum_n{E_{b}}$, for each block $n$:
\begin{equation}
    E_{b}=T\times(fr\times{E_{AC}}\times{OP_{AC}}+E_{MAC}\times{OP_{MAC}})\label{energycost}
\end{equation}
where $T$ and $fr$ represents the total time steps and the block firing rate. The blocks are normally convolutional or fully connected, and the energy consumption is determined by the number of AC  and MAC  operations ($OP_{AC}, OP_{MAC}$). In this work, we adopt the same structure of SNN and ANN to compare the energy consumption and assume that the data for various operations are 32-bit floating-point implementation in 45nm technology~\cite{horowitz20141}, where $E_{MAC}=4.6pJ$ and $E_{AC}=0.9pJ$.
\section{Method}


\subsection{Input Representation}
\paragraph{Static Images Inputs}

Typically, considering the spatio-temporal feature of SNNs, the static images generated by the frame cameras are copied and used as the input frame for each time step~\cite{hu2021advancing,fang2021deep,10032591}.
\paragraph{Event-based Inputs}
 Event cameras work completely differently from frame cameras where each pixel responds independently to changes in light. An event $e_n=(x_n,y_n,t_n,p_n)$ is generated for a pixel $(x_n,y_n)$ at the timestamp $t_n$ when the logarithmic light change $I(x,y,t)$ exceeds the threshold $\theta_{th}$. The polarity $p_n\in\{-1,1\}$ denotes the increase or decrease of light intensity.

Given the spatio-temporal window $\zeta$, the asynchronous event stream $E=\{e_n\in\zeta:n=1,...,N\}$ represents a sparse grid of points in 3D space. In this work, we split $E$ into temporal bins with a constant temporal window $dt$, which maps the events into image-like 2D representations~\cite{yao2021temporal}. The network processes $T$ fixed time steps each time, and the total sequence $\Gamma=T\times{dt}$.

\subsection{The Energy-Efficient Resisual Block}

\begin{figure}[!t]
\centering
\includegraphics[scale=0.46]{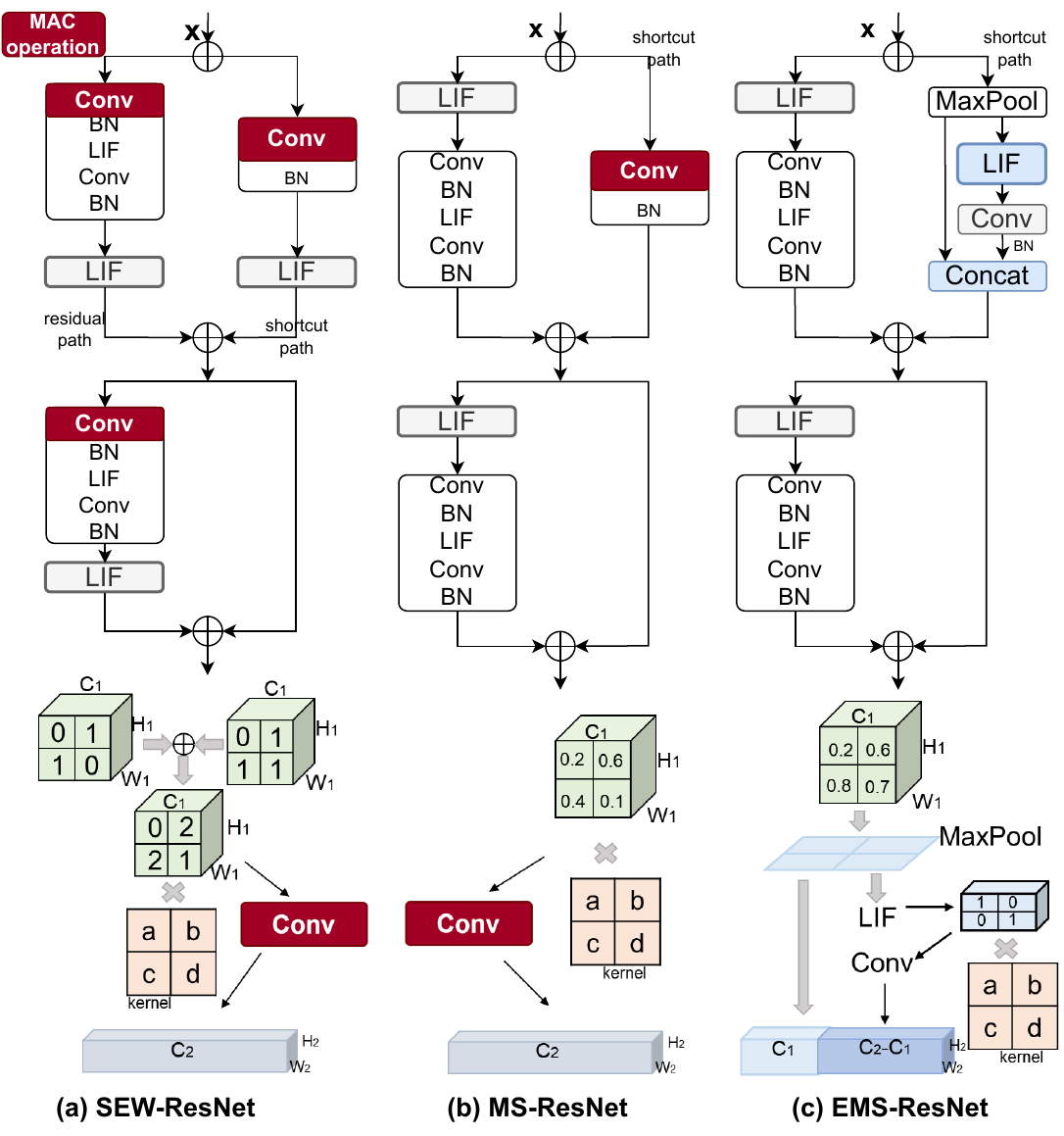}
\caption{{\bf{The Sew-ResNet, the MS-ResNet and proposed EMS-ResNet}.} (a) The sum of spikes in SEW-ResNet causes non-spike convolution operations. (b) MS-ResNet introduces non-spike convolution on shortcut paths where the dimensionality and number of channels changes. (c) The full-spike EMS-ResNet.}
\label{fig:res_compare}
\end{figure}

Currently, the main structures used for deep directly training in SNNs are SEW-ResNet and MS-ResNet. Here, we define the output of the L-th residual block, denoted by ${\bf{X}}^{L}$, as:
\begin{equation}
    {\bf{X}}^L=Add({\bf{F}}^r({\bf{X}}^{L-1}),{\bf{F}}^s({\bf{X}}^{L-1}))
\end{equation}
where the residual path is ${\bf{F}}^r(\cdot)$, and the shortcut path is represented as ${\bf{F}}^s(\cdot)$. As shown in Figure~\ref{fig:res_compare}a, SEW-ResNet implements residual learning by restricting the final LIF activation function to each of the residual and shortcut paths. When both paths transmit spikes, $Add({\bf{F}}^r({\bf{X}}^{L-1}),{\bf{F}}^s({\bf{X}}^{L-1}))=2$, which means that the non-spike convolution operation in the next block will introduce MAC operations. Although using either AND or IAND operation to avoid this problem has been tried, it comes with an unbearable  performance loss. For the MS-ResNet in the Figure~\ref{fig:res_compare}b, it is apparent that it ignores the non-spike convolution operations on shortcut paths. When the dimensionality of the network or the number of channels is constantly changing, the energy consumption caused by this component is not negligible. According to Equation~\ref{energycost}, to achieve energy efficiency, we design a full-spike residual block (EMS-Block) from the perspective of reducing parameters, and avoiding MAC operations.

 As illustrated in Figure~\ref{fig:res_compare}c, our proposed EMS-ResNet is designed differently for channel number and dimension variation. The idea of MS-ResNet is adopted on the residual path. On the shortcut path, when the number of channels changes, an additional LIF is added before the convolution for converting the information into sparse spikes which enables the entire network can be fully spiked.  We transform the EMS-Block into two formats (see Figure~\ref{fig:ODSNN}). The EMS-Block1 is used when the channel number is constant or decreasing. For increasing channel numbers, we design the EMS-Block2 that uses concatenation operations to feature reuse while reducing parameters. Considering features are usually extracted by constant downsampling in the object detection task, we use the maxpool to reduce parameters. Our design improves flow of information and the output flow on the shortcut path is conceptually approximate to the sum of synaptic inputs to neuronal membranes. The full spike feature implies energy efficient properties for the network, so we name our network as the Energy-Efficient Menbrane-Shortcut ResNet (EMS-ResNet). 
 

\subsection{The EMS-YOLO Model}
\begin{figure*}[!t]
\centering
\includegraphics[scale=0.45]{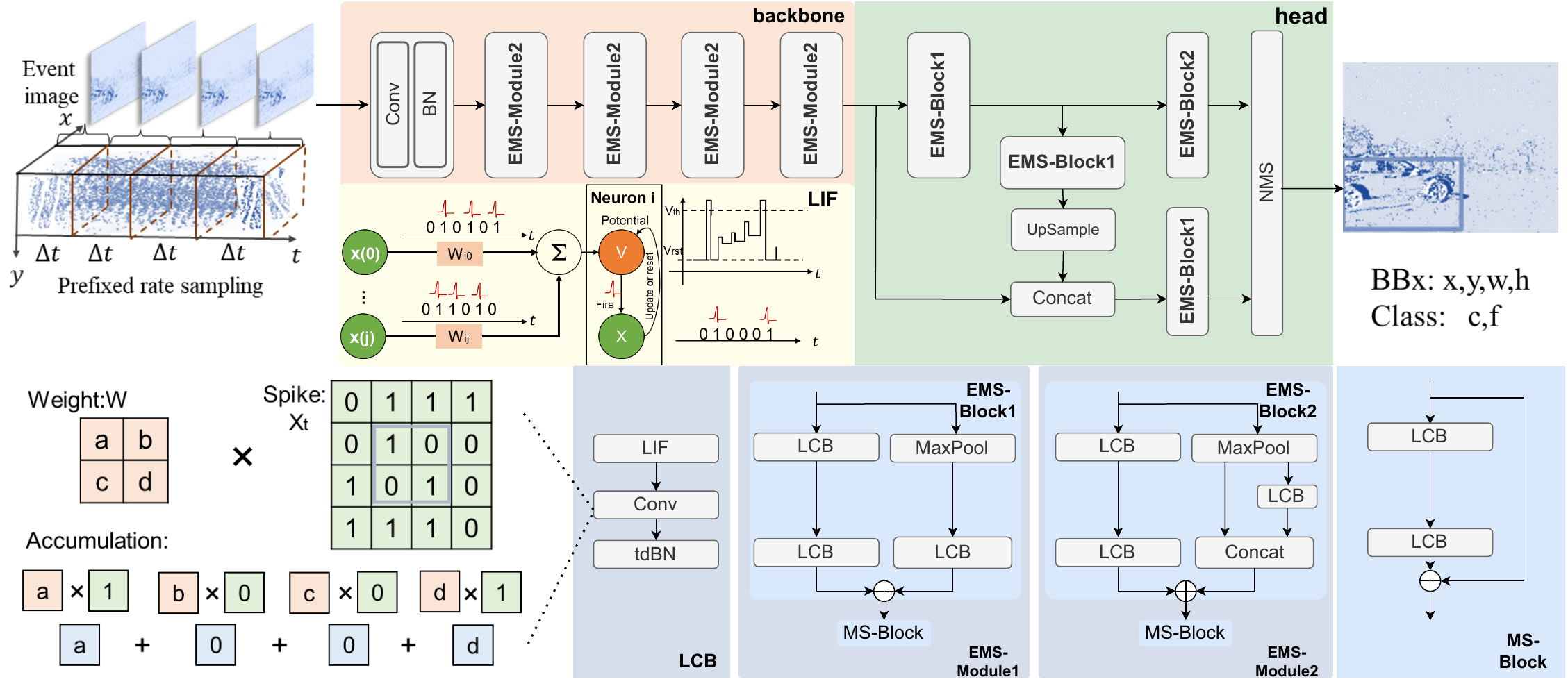}
\caption{{\bf{The proposed directly-trained SNN for object detection (EMS-YOLO)}}. EMS-YOLO mainly consists of backbone and head, which are mainly composed of EMS-Blocks. EMS-Module1 and EMS-Module2 are EMS-Block1 and MS-Block, EMS-Block2 and MS-Block connections respectively. EMS-Block2 is used when the number of output channels increases, otherwise, EMS-Block1 is used.}
\label{fig:ODSNN}
\end{figure*}

Our goal is to predict the classification and position of the objects from the given static image or event stream that can be represented as ${\bf{X}}=\{X_t\}^T_{t=1}$. The dimensions of the input data $X_t$ for each time step $t$ are $C\times H\times W$, where $C$ is the channel number, $H\times W$ is the resolution. To achieve this, we calculate the information of N objects ${\bf{B}}=\{B_n\}^N_{n=1}$ by:

\begin{equation}
    \bf{B}=\mathcal{D}(\bf{X})
\end{equation}
where each $B_n=\{x_n,y_n,w_n,h_n,c_n,f_n\}$ is the corresponding object's bounding box position and class prediction. $(x_n,y_n),w_n,h_n$ represent the upper-left spatial coordinates, width, and length of the bounding box, respectively. $c_n,f_n$ are the class and confidence level of the object, respectively. $\mathcal{D}$ refers to our proposed EMS-YOLO, which could achieve performance comparable to that of a same-structure ANN, while inferring in real time.

As shown in Figure~\ref{fig:ODSNN}, our EMS-YOLO is a variation of the YOLO framework, consisting of two modules: the backbone for feature extraction and the detection heads. The analog pixel values $\bf{X}$ of an image are directly applied to the input layer of EMS-YOLO. For the backbone, the first convolutional layer is trained to convert inputs into spikes where LIF neurons integrate the weighted inputs and generate output spike trains when the membrane potential exceeds the trained firing threshold. Then we employ the EMS-Modules to extract object features from different dimensions and number of channels, which can enhance the robustness of the network. The number of EMS-Blocks and the channel width can be dynamically adjusted to the particular task, that Ablation experiments are shown in Sec~\ref{ablations}, and EMS-ResNet18 is shown in the Figure~\ref{fig:ODSNN} as an example.

For the detection head, to avoid the loss of SNN performance due to the multi-layer direct-connected convolution structure~\cite{zheng2021going} of the conventional detection head, in this work we take the yolov3-tiny detection head~\cite{redmon2018yolov3} as an example and replace multiple directly connected convolutions with EMS-Blocks.

The main challenge of object detection, as a regression task with SNN models, is to convert the features extracted with spike trains into accurate continuous value representations of the bounding boxes coordinates. Here, we feed last membrane potential of the neurons ~\cite{zhu2022event} into each detector to generate anchors of different sizes. After NMS processing, the final category and bounding box coordinates of different objects can be obtained. Our end-to-end EMS-YOLO is trained with the cross-entropy loss function\footnote{https://github.com/ultralytics/yolov3}.

\subsection{Analysis of Gradient Vanishing/Explosion Problems}\label{sec.gradient}
To explore the potential of our EMS-ResNet to enable deep training, we analyze the feasibility from a theoretical perspective. According to Gradient Norm Equality (GNE) theory~\cite{chen2020tpami}, our EMS-ResNet could avoid exploding and vanishing gradients.

\begin{lemma}
Consider a neural network that can be represented as a series of blocks as  and the $j$th block’s
jacobian matrix is denoted as $J_j$. If $\forall j; \phi(J_jJ_j^\mathsf{T}) \approx 1$ and
$\varphi(J_jJ_j^\mathsf{T}) \approx 0$ , the network achieves “Block Dynamical
Isometry” and can avoid gradient vanishing or explosion.
\end{lemma}

Here, ${J}_{j}$ means the Jacobian matrix of the block $j$, $j$ is the index of the corresponding block. $\phi$ means the expectation of the normalized trace. $\varphi$ means $\phi(\boldsymbol{A}^{2})-\phi^{2}(\boldsymbol{A})$. 

Briefly, the theory ensures the gradient of the network will not decrease to $0$ or explode to $\infty$ since every block have $\phi(J_jJ_j^\mathsf{T}) \approx 1$. And $\varphi(J_jJ_j^\mathsf{T}) \approx 0$ makes sure that accident situation won't happen. And in most cases\cite{zheng2021going}\cite{hu2021advancing}, $\phi(J_jJ_j^\mathsf{T}) \approx 1$ is enough for avoiding gradient vanish or exploding. Detailed description of the notation and the theory are in \cite{chen2020tpami}.

\begin{table*}[t]
\begin{center}
    
    \begin{tabular}{c c c c c }
    \hline
    Method &Work & Model &Time Step(T) &mAP@0.5\\
    \hline
    \multirow{3}{*}{Backpropagation}&\multirow{3}{*}{ANN}&Tiny-Yolo&{/}&{0.262} \\
    \multirow{3}{*}{}&\multirow{3}{*}{}&ResNet50&{/}&{0.460} \\
    \multirow{3}{*}{}&\multirow{3}{*}{}&ResNet34&{/}&{0.565} \\
    \hline
    \multirow{7}{*}{\thead{ANN-SNN\\ \\ Conversion}}&Spiking-Yolo~\cite{kim2020spiking}&\multirow{3}{*}{Tiny-Yolo}&3500&0.257 \\
  
    \multirow{7}{*}{}&\multirow{2}{*}{Bayesian Optimization~\cite{kim2020towards}}&\multirow{3}{*}{}&500&0.211 \\
    
    \multirow{7}{*}{}&\multirow{3}{*}{}&\multirow{2}{*}{}&5000&0.258 \\
    \cline{2-5}
    \multirow{7}{*}{}&\multirow{4}{*}{Spike Calibration~\cite{li2022spike}}&\multirow{4}{*}{\thead{ResNet50+Burst+ \\ MLIPooling+SpiCalib} }&64&0.331 \\
    \multirow{7}{*}{}&\multirow{4}{*}{}&\multirow{4}{*}{}&128&0.436 \\
    \multirow{7}{*}{}&\multirow{4}{*}{}&\multirow{4}{*}{}&256&0.453 \\
    \multirow{7}{*}{}&\multirow{4}{*}{}&\multirow{4}{*}{}&512&0.454 \\
    \hline
    {\bf{\thead{Directly-Trained \\SNN}}}&{EMS-YOLO}&{EMS-ResNet34}&{\bf{4}}&{\bf{0.501}}\\
    \hline
    \end{tabular}
    \vspace{1mm}
    \caption{ Results on the COCO2017 DataSet.}
    \vspace{-5mm}
    \label{tab:coco1}
    \end{center}
\end{table*}

\begin{definition}
\textbf{(General Linear Transform) }Let $\boldsymbol{f}(\boldsymbol{x})$ be a transform whose Jacobian matrix is $\boldsymbol{J}$. f is called general linear transform when it satisfies:

\begin{equation}
    E\left[\frac{\|\boldsymbol{f}(\boldsymbol{x})\|_{2}^{2}}{\operatorname{len}(\boldsymbol{f}(\boldsymbol{x}))}\right]=\phi\left(\boldsymbol{J J}^{T}\right)E\left[\frac{\|\boldsymbol{x}\|_{2}^{2}}{\operatorname{len}(\boldsymbol{x})}\right] .
\end{equation}
\end{definition}

According to the random matrix and mean field theory \cite{poole2016exponential}, the data flow propagating in the network can be regarded as random variables. Thus, $\boldsymbol{x}$ is considered as a random variable. We denote the notation $E\left[\frac{\|\boldsymbol{x}\|_{2}^{2}}{\operatorname{len}(\boldsymbol{x})}\right]$ as the $2^{th}$ moment of input element. 

The definition is useful for analysis the gradient. Because once the output of an EMS-Block is normalized by the BN layer, the $2^{th}$ moment $E\left[\frac{\|\boldsymbol{f}(\boldsymbol{x})\|_{2}^{2}}{\operatorname{len}(\boldsymbol{f}(\boldsymbol{x}))}\right]$ which we  denote as $\alpha_{2}$ is clear.

\begin{lemma}[Multiplication](Theorem 4.1 in~\cite{chen2020tpami})
Given $\boldsymbol{J}:=\prod_{j=L}^1 \boldsymbol{J}_j$, where $\{\boldsymbol{J}_j\in \mathbb{R}^{m_j\times m_{j-1}}\}$ is a series of independent random matrices. If $(\prod_{j=L}^1\boldsymbol{J}_j)(\prod_{j=L}^1\boldsymbol{J}_j)^T$ is at least the $1^{st}$ moment unitarily invariant, we have 
\begin{equation}
    \phi\left((\prod_{j=L}^1\boldsymbol{J}_j)(\prod_{j=L}^1 \boldsymbol{J}_j)^T \right)=\prod_{j=L}^1\phi(\boldsymbol{J}_j {\boldsymbol{J}_j}^T).
\end{equation}
\label{lemma:multiplication}
\end{lemma}

\begin{lemma}[Addition] (Theorem 4.2 in \cite{chen2020tpami}) Given $\boldsymbol{J}:=\prod_{j=L}^1 \boldsymbol{J}_j$, where $\{\boldsymbol{J}_j\in \mathbb{R}^{m_j\times m_{j-1}}\}$ is a series of independent random matrices. If at most one matrix in $\boldsymbol{J}_j$ is not a central matrix, we have
\begin{equation}
    \phi(\boldsymbol{JJ}^T)=\sum_j\phi(\boldsymbol{J}_j {\boldsymbol{J}_j}^T).
\end{equation}
\label{lemma:addition}
\end{lemma}

The multiplication and addition principle provide us a technique for analysising the network with serial connections and parallel connections.

\textbf{Discussion for General Linear Transform.} Since the Jacobian matrix of pooling can be regarded as a matrix $\boldsymbol{J}$ that the element of matrix $\left[\boldsymbol{J}\right]_{ik}\in\{0, 1\}$, $0$ for the element not chosen while $1$ for element chosen. Thus the pooling layer can be regarded as a general linear transform. Similarly, the UpSampling layer in the detection head is also a generalized linear transformation. Concatenation is general linear transform too, because the function of $\boldsymbol{f(x)}=[\boldsymbol{x}, \Tilde{\boldsymbol{f(x)}}]$ can be expressed as $\boldsymbol{f(x)}=[\boldsymbol{I} \quad \Tilde{\boldsymbol{J}}]\boldsymbol{x}$ if the function $\Tilde{\boldsymbol{f}}$ is general linear transform. Since the BN and CONV layers are already discussed in~\cite{chen2020tpami}, so we only need to assume that the LIF layer satisfy general linear transformation, which is already used in proof by \cite{zheng2021going}. Since the EMS-ResNet is a serial of basic EMS-Blocks and MS-Blocks, we can separately analysis those blocks and multiply them together. 

\begin{proposition}
For EMS-Block1 and EMS-Block2, the Jacobian matrix of the block can be represented as $\phi(\boldsymbol{J _{j}J_{j}^{T}})=\frac{2}{\alpha_{2}^{j-1}}$.
\end{proposition}

\begin{proof}
It is detailed in \textbf{Supplementary Material Proof A.1 and Proof A.2}.
\end{proof} 

\begin{proposition}
\label{goodproperty}
For the EMS-ResNet, $\phi(\boldsymbol{JJ^{T}})\approx 1$ can be satisfied by control the $2^{th}$ moment of the input.
\end{proposition}


\begin{proof}
It is detailed in \textbf{Supplementary Material Proof A.3}.
\end{proof} 
\textbf{Discussion about the proposition.} According to \cite{chen2020tpami}, $\phi(\boldsymbol{J}\boldsymbol{J}^{T})\approx1$ is the formal expression for GNE which ensures gradient vanishing or explosion will not happen. However, it is enough for solving gradient problems by just avoiding the exponential increase or decrease~\cite{chen2020tpami}. In our backbone, the MS-Block, the only factor to evoke exponential increase of gradient, whose increasing is stopped by the interspersing of EMS-Blocks. Even if the initialize of BN do not satisfy the constant for $\phi(\boldsymbol{J}\boldsymbol{J}^{T})\approx1$, the gradient of each block will not increase as the network goes deeper. Briefly, we improve the property of network at a structural level by our EMS-Blocks.

\section{Experiments}
To fully validate the effectiveness of our model, we conduct experiments on the frame-based COCO2017 Dataset~\cite{lin2014microsoft} and the event-based Gen1 Dataset~\cite{de2020large}. The detailed introduction of the Datasets are in the \textbf{Supplementary Material B}. For the object detection task, the most widely used model evaluation metric is the mean Average Precision (mAP). We report the mAP at IOU=0.5 (mAP@0.5) and the average AP between 0.5 and 0.95 (mAP@0.5:0.95). 
\subsection{Experimental Setup}\label{experiment_setup}
In all experiments, we set the reset value $V_{reset}$ of LIF neurons to 0, the membrane time constant $\tau$ to 0.25, the threshold $V_{th}$ to 0.5, and the coefficient $\alpha$ to 1. We mainly train models on 8 NVIDIA RTX3090 GPUs and adopt the SGD optimizer, the learning rate is set to $1e^{-2}$. The network is trained for 300 epochs on the COCO2017 Dataset with a batch size of 32. On the Gen1 Dataset, we trained the model for 100 epochs, with the batch size of 128. 

\subsection{Effective Test}

\paragraph{COCO2017 Dataset} Our models are trained on Mosic data augmentation techniques~\cite{bochkovskiy2020yolov4} and tested to detect 80 objects of the COCO Dataset~\cite{lin2014microsoft} on the validation set. As the first one to implement object detection task with directly trained SNNs, the main quantitative results are presented in Table~\ref{tab:coco1}. For comparison with the current best performance model Spike Calibration~\cite{li2022spike} based on ANN-SNN conversion method, we do experiments based on EMS-ResNet34, which has a reasonable number of parameters compared with ResNet50. 
We demonstrate that the directly training method can achieve higher performance at only 4 time steps, while the ANN-SNN conversion approach requires at least a several hundred time steps. Furthermore, with the same experimental settings and data augmentation methods, the performance of our model is comparable to that of the ANN with the same structure while reducing \textbf{5.83$\times$} of energy consumption.

\begin{table}[htbp]\footnotesize
  \centering
 
  \setlength{\tabcolsep}{1mm}
    \begin{tabular}{ccccccc}
    \toprule
    \multirow{2}[2]{*}{Method} & \multirow{2}[2]{*}{Model} & \multirow{2}[2]{*}{Params} & \multirow{2}[2]{*}{T} & Firing & mAP   & mAP \\
          &       &       &       & Rate  & @0.5:0.95 & @0.5 \\
    \midrule
    \multirow{6}[2]{*}{Sew-Reset} & MobileNet & \multirow{2}[1]{*}{12.64M} & \multirow{2}[1]{*}{5} & \multirow{2}[1]{*}{22.22\%} & \multirow{2}[1]{*}{0.174} & \multirow{2}[1]{*}{/} \\
          & -64+SSD &       &       &       &       &  \\
          & DenseNet & \multirow{2}[0]{*}{24.26M} & \multirow{2}[0]{*}{5} & \multirow{2}[0]{*}{29.44\%} & \multirow{2}[0]{*}{0.147} & \multirow{2}[0]{*}{/} \\
          & 121-24+SSD &       &       &       &       &  \\
          & VGG   & \multirow{2}[1]{*}{8.20M} & \multirow{2}[1]{*}{5} & \multirow{2}[1]{*}{37.20\%} & \multirow{2}[1]{*}{0.189} & \multirow{2}[1]{*}{/} \\
          & -11+SSD &       &       &       &       &  \\
    \midrule
    ANN   & ResNet10 & \multirow{2}[2]{*}{\textbf{6.20M}} & /     & /     & 0.247 & 0.504 \\
    EMS-ResNet & EMS-Res10 &       & 5     & \textbf{21.15\%} & \textbf{0.267} & \textbf{0.547} \\
    \bottomrule
    \end{tabular}%
    \vspace{1mm}
    \caption{Results on Gen1 Dataset.}
    \vspace{-5mm}
    \label{tab:gen1data}
\end{table}%

\paragraph{GEN1 Automotive Detection Dataset} As the largest event camera-based dataset, it contains two categories (pedestrians and cars)~\cite{de2020large}. Currently, only Cordone \etal.~\cite{cordone2022object} have experimented with the SNN method so far. They divided the event data into five temporal bins for each input to the network. In this paper, all the experiments on the Gen1 Dataset follow the same experimental setup with them.

The backbone of their shallow model is based on the lightweight structure of Sew-ResNet, and the detection head uses SSD. For fair comparison, we train a model with only 6.20 M parameters based on the EMS-ResNet10 structure while reducing the number of channels. The results in Table~\ref{tab:gen1data} show that our method can achieve the mAP@0.5:0.95 of 0.267 at the same number of time steps, which far outperforms their performance of 0.189. It is remarkable that our full-spike model is more sparse, with a spike firing rate of 21.15\%, which implies lower energy consumption. We also conducted comparison experiments with an ANN-ResNet10~\cite{gong2022resnet10} of the same structure with reduced number of channels and discovered we even achieved better performance than the ANN network, which indicates that the SNN network may be more competitive in processing event data.

\begin{figure}[h]
\centering
\includegraphics[scale=0.27]{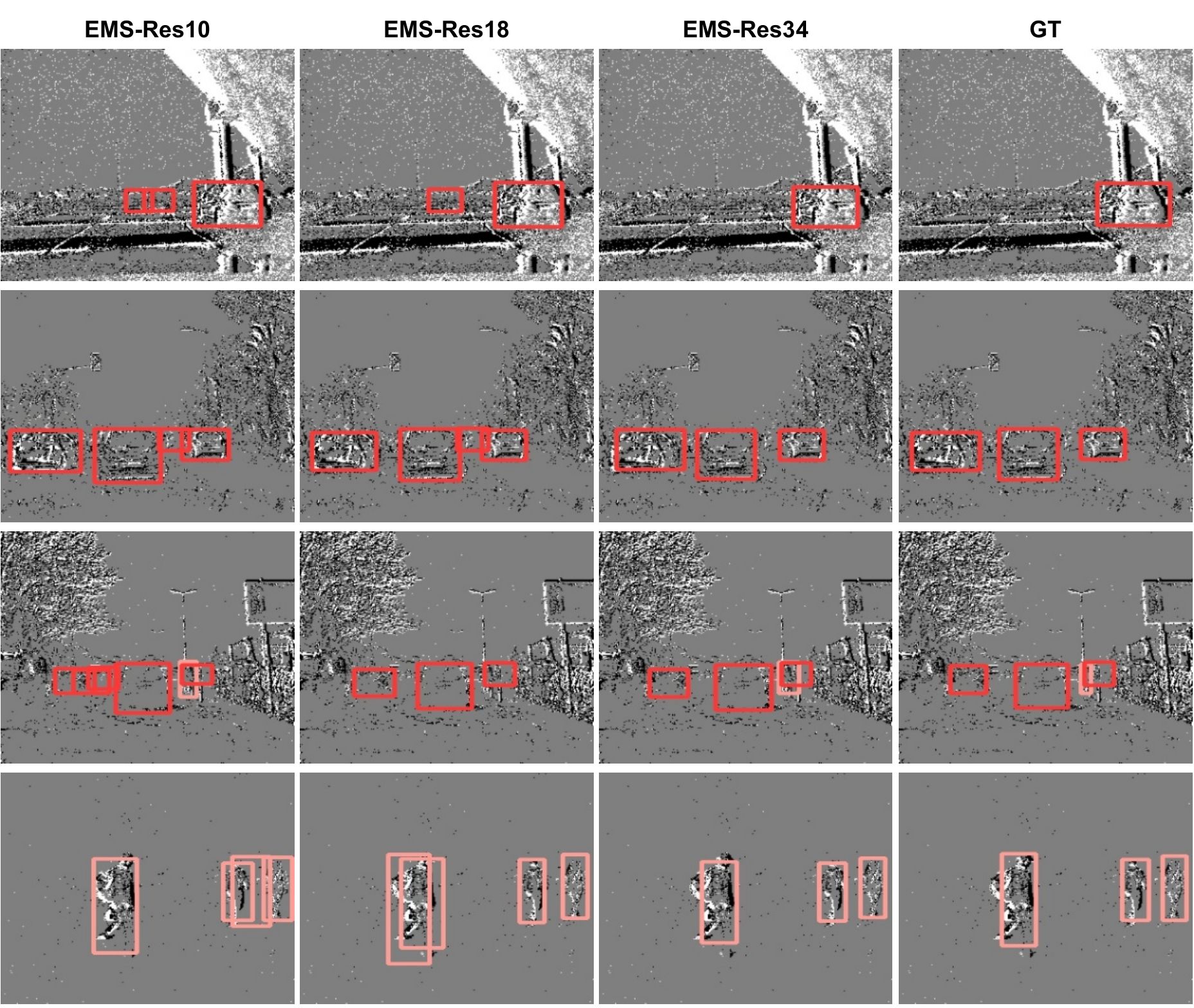}
\caption{\textbf{Object detection results on the Gen1 Dataset.} } 
\label{fig:Gen1_result}
\end{figure}

\subsection{Ablation Experiments}\label{ablations}
We conducted ablation experiments to better understand the effectiveness of different residual modules on reducing energy consumption, the effect of network depth and time step on model performance. All ablation experiments are performed by following the experimental setup in Sec~\ref{experiment_setup}, if not otherwise specified.

\begin{table}[h]
    
    \small
    \centering
    \resizebox{240pt}{35pt}{
    \begin{tabular}{cccccc}
    \hline
    {Model} &{\thead{mAP\\@0.5}} &{\thead{mAP\\@0.5:0.95}}&{Params} &\thead{Firing\\ Rate}&\thead{Energy\\Efficiency}\\
    \hline
    {ANN-Res18}&{0.537}&{0.290}&{9.56M}&{/}&{1$\times$}\\
    {MS-Res18}&{0.560}&{0.285}&{9.49M}&{17.08\%}&{2.43$\times$}\\
    {Sew-Res18}&{0.561}&{0.286}&{9.56M}&{18.80\%}&{2.00$\times$*}\\
    {EMS-Res18}&{\textbf{0.565}}&{0.286}&{\bf{9.34M}}&{20.09\%}&{\bf{4.91$\times$}}\\
    \hline
    \end{tabular}
    }
    \vspace{1mm}
    \caption{Impact of different residual blocks on Gen1 Dataset. *Non-spike convolution blocks are calculated as MAC operations.}
    \vspace{-4mm}
    \label{tab:diff_block}
\end{table}

\begin{figure*}[!t]
\centering
\includegraphics[scale=0.55]{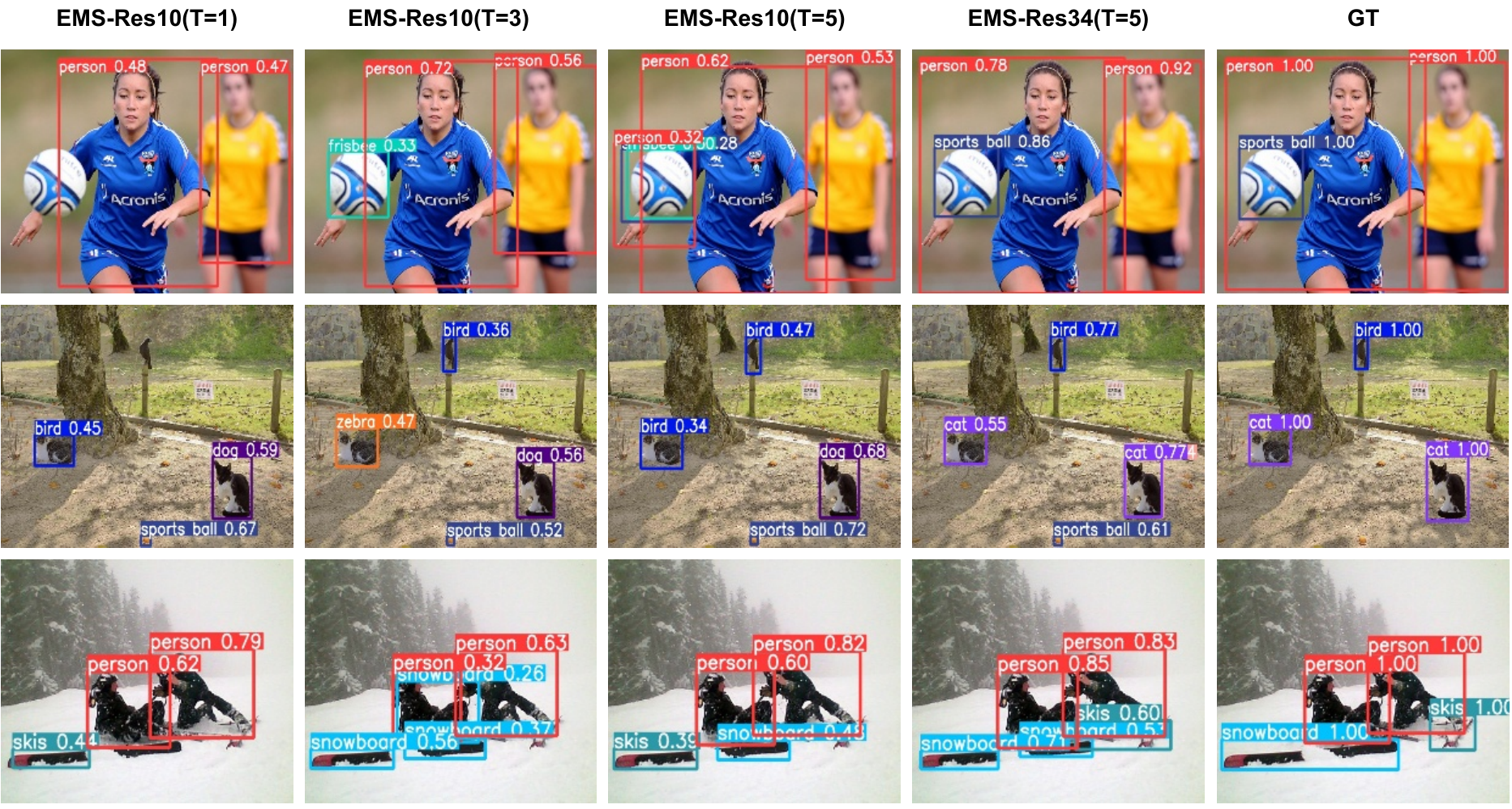}
\caption{{\bf{Object detection results on the COCO Dataset}}. The first three columns compare the effect of time steps on performance for the same network structure. The third and fourth columns compare the effect of the depth of the network on performance.}
\label{fig:COCO_result}
\end{figure*}

\paragraph{Different Residual Blocks}
To explore the structural superiority of our EMS-Block, we made comparative experiments based on the Gen1 Dataset (shown in Table~\ref{tab:diff_block}). All experiments are trained with only 50 epochs, and the batchsize is set to 64. From the experimental results, we found that the performance of our full-spike EMS-ResNet is comparable to those of the other two ResNets with non-spikes and the sparsity of the network is also maintained. The full-spike property of EMS-ResNet enables the network to be energy efficiency. According to Equation~\ref{energycost}, we calculated the energy consumption of the network, excluding the convolutional energy consumption from the first coding layer. The energy consumption on ANN-ResNet18 is around $9.65mJ$, which we denote this baseline energy consumption as $1\times$. Our EMS-Res18 can reduce up to $4.91\times$ of energy consumption.

\begin{table}[h]
    \small
    \centering
    \begin{tabular}{ccccc}
    \hline
    {Model} &\thead{mAP\\@0.5} &\thead{mAP\\@0.5:0.95}&{Params}&\thead{Firing\\ Rate} \\ 
    \hline
    {EMS-Res10}&{0.547}&{0.267}&{6.20M}&{21.15\%}\\
    {EMS-Res18}&{0.565}&{0.286}&{9.34M}&{20.09\%}\\
    {EMS-Res34}&{0.590}&{0.310}&{14.40M}&{17.80\%}\\
    \hline
    
    \end{tabular}
    \vspace{1mm}
    \caption{Ablation studies of different numbers of residual blocks on Gen1 Dataset.}
    \vspace{-7mm}
    \label{tab:num_blocks}
\end{table}

\paragraph{Numbers of Residual Blocks} 
In Sec~\ref{sec.gradient}, we theoretically analyze that our EMS-ResNet can achieve deep training. Here we report results in Table~\ref{tab:num_blocks} on the Gen1 Dataset based on EMS-Res10, EMS-Res18, and EMS-Res34, respectively. When the scale of the network is larger, the feature extraction ability becomes stronger (see Figure~\ref{fig:Gen1_result}).

\paragraph{Size of Time Steps}
Due to the sparsity of event streams, different event sampling strategies may affect the ablation experiment of time steps on Gen1 Dataset. Thus, we report the performance of EMS-ResNet10 based on the COCO2017 Dataset for T = 1, 3, 5, 7 in Table~\ref{tab:size of time}. The time steps can be dynamically adjusted to achieve a balance of effectiveness and efficiency according to the needs of the actual task. We show some detection results on the test set compared in Figure~\ref{fig:COCO_result}, where it can be found that the accuracy of object detection is higher when the time step is longer. Here, we first train the model based on T=1, and then use the training model with T=1 as a pre-training model with multiple time steps, which can effectively reduce the training time.

\begin{table}[h]
    \small
    \centering
    {
    \begin{tabular}{ccccc}
    \hline
    {T} &{1} &{3}&{5}&{7}\\ 
    \hline
    {mAP@0.5}&{0.328}&{0.362}&{0.367} &{0.383}\\
    {mAP@0.5:0.95}&{0.162}&{0.184}&{0.189} &{0.199}\\
 
    \hline
    
    \end{tabular}
    }
    \vspace{1mm}
    \caption{Impact of different size of time steps on COCO Dataset.}
    \vspace{-5mm}
    \label{tab:size of time}
\end{table}

\section{Conclusion}
In this work, we are the first to use the deep directly trained SNN for the object detection task. Considering that object detection involves extracting multi-scale object features, we design a novel energy-efficient full-spike residual block, EMS-ResNet, that eliminates the redundant MAC operations generated by non-spike convolution on shortcut paths and residual connections. The full-spike EMS-ResNet makes it easier to deploy on neuromorphic chips. Moreover, our results show that our model (EMS-YOLO) can achieve comparable performance to that of a same-structure ANN in a very short time step and performs well on both static images and event data. We believe that our model will drive the exploration of SNN for various regression tasks while becoming possible in neuromorphic hardware deployment.

{\small
\bibliographystyle{ieee_fullname}
\bibliography{egbib}

\begin{thebibliography}{10}\itemsep=-1pt

\bibitem{abbott1999lapicque}
Larry~F Abbott.
\newblock Lapicque’s introduction of the integrate-and-fire model neuron
  (1907).
\newblock {\em Brain research bulletin}, 50(5-6):303--304, 1999.

\bibitem{akopyan2015truenorth}
Filipp Akopyan, Jun Sawada, Andrew Cassidy, Rodrigo Alvarez-Icaza, John Arthur,
  Paul Merolla, Nabil Imam, Yutaka Nakamura, Pallab Datta, Gi-Joon Nam, et~al.
\newblock Truenorth: Design and tool flow of a 65 mw 1 million neuron
  programmable neurosynaptic chip.
\newblock {\em IEEE transactions on computer-aided design of integrated
  circuits and systems}, 34(10):1537--1557, 2015.

\bibitem{balasubramaniam2022object}
Abhishek Balasubramaniam and Sudeep Pasricha.
\newblock Object detection in autonomous vehicles: Status and open challenges.
\newblock {\em arXiv preprint arXiv:2201.07706}, 2022.

\bibitem{bochkovskiy2020yolov4}
Alexey Bochkovskiy, Chien-Yao Wang, and Hong-Yuan~Mark Liao.
\newblock Yolov4: Optimal speed and accuracy of object detection.
\newblock {\em arXiv preprint arXiv:2004.10934}, 2020.

\bibitem{cao2015spiking}
Yongqiang Cao, Yang Chen, and Deepak Khosla.
\newblock Spiking deep convolutional neural networks for energy-efficient
  object recognition.
\newblock {\em International Journal of Computer Vision}, 113(1):54--66, 2015.

\bibitem{chen2020tpami}
Zhaodong Chen, Lei Deng, Bangyan Wang, Guoqi Li, and Yuan Xie.
\newblock A comprehensive and modularized statistical framework for gradient
  norm equality in deep neural networks.
\newblock {\em IEEE Transactions on Pattern Analysis and Machine Intelligence},
  44(1):13--31, 2022.

\bibitem{cordone2022object}
Lo{\"\i}c Cordone, Beno{\^\i}t Miramond, and Philippe Thierion.
\newblock Object detection with spiking neural networks on automotive event
  data.
\newblock In {\em 2022 International Joint Conference on Neural Networks},
  pages 1--8. IEEE, 2022.

\bibitem{davies2018loihi}
Mike Davies, Narayan Srinivasa, Tsung-Han Lin, Gautham Chinya, Yongqiang Cao,
  Sri~Harsha Choday, Georgios Dimou, Prasad Joshi, Nabil Imam, Shweta Jain,
  et~al.
\newblock Loihi: A neuromorphic manycore processor with on-chip learning.
\newblock {\em Ieee Micro}, 38(1):82--99, 2018.

\bibitem{de2020large}
Pierre de Tournemire, Davide Nitti, Etienne Perot, Davide Migliore, and Amos
  Sironi.
\newblock A large scale event-based detection dataset for automotive.
\newblock {\em arXiv preprint arXiv:2001.08499}, 2020.

\bibitem{deng2020rethinking}
Lei Deng, Yujie Wu, Xing Hu, Ling Liang, Yufei Ding, Guoqi Li, Guangshe Zhao,
  Peng Li, and Yuan Xie.
\newblock Rethinking the performance comparison between snns and anns.
\newblock {\em Neural networks}, 121:294--307, 2020.

\bibitem{diehl2015fast}
Peter~U Diehl, Daniel Neil, Jonathan Binas, Matthew Cook, Shih-Chii Liu, and
  Michael Pfeiffer.
\newblock Fast-classifying, high-accuracy spiking deep networks through weight
  and threshold balancing.
\newblock In {\em 2015 International joint conference on neural networks},
  pages 1--8. ieee, 2015.

\bibitem{fang2021deep}
Wei Fang, Zhaofei Yu, Yanqi Chen, Tiejun Huang, Timoth{\'e}e Masquelier, and
  Yonghong Tian.
\newblock Deep residual learning in spiking neural networks.
\newblock {\em Advances in Neural Information Processing Systems},
  34:21056--21069, 2021.

\bibitem{girshick2014rich}
Ross Girshick, Jeff Donahue, Trevor Darrell, and Jitendra Malik.
\newblock Rich feature hierarchies for accurate object detection and semantic
  segmentation.
\newblock In {\em Proceedings of the IEEE conference on computer vision and
  pattern recognition}, pages 580--587, 2014.

\bibitem{gong2022resnet10}
Jiaming Gong, Wei Liu, Mengjie Pei, Chengchao Wu, and Liufei Guo.
\newblock Resnet10: A lightweight residual network for remote sensing image
  classification.
\newblock In {\em 2022 14th International Conference on Measuring Technology
  and Mechatronics Automation}, pages 975--978. IEEE, 2022.

\bibitem{hodgkin1952quantitative}
Alan~L Hodgkin and Andrew~F Huxley.
\newblock A quantitative description of membrane current and its application to
  conduction and excitation in nerve.
\newblock {\em The Journal of physiology}, 117(4):500, 1952.

\bibitem{horowitz20141}
Mark Horowitz.
\newblock 1.1 computing's energy problem (and what we can do about it).
\newblock In {\em 2014 IEEE International Solid-State Circuits Conference
  Digest of Technical Papers}, pages 10--14. IEEE, 2014.

\bibitem{hu2021advancing}
Yifan Hu, Lei Deng, Yujie Wu, Man Yao, and Guoqi Li.
\newblock Advancing spiking neural networks towards deep residual learning.
\newblock {\em arXiv preprint arXiv:2112.08954}, 2021.

\bibitem{hu2021spiking}
Yangfan Hu, Huajin Tang, and Gang Pan.
\newblock Spiking deep residual networks.
\newblock {\em IEEE Transactions on Neural Networks and Learning Systems},
  2021.

\bibitem{izhikevich2003simple}
Eugene~M Izhikevich.
\newblock Simple model of spiking neurons.
\newblock {\em IEEE Transactions on neural networks}, 14(6):1569--1572, 2003.

\bibitem{johansson2021training}
Olof Johansson.
\newblock Training of object detection spiking neural networks for event-based
  vision, 2021.

\bibitem{kaiser2020synaptic}
Jacques Kaiser, Hesham Mostafa, and Emre Neftci.
\newblock Synaptic plasticity dynamics for deep continuous local learning
  (decolle).
\newblock {\em Frontiers in Neuroscience}, 14:424, 2020.

\bibitem{kim2020towards}
Seijoon Kim, Seongsik Park, Byunggook Na, Jongwan Kim, and Sungroh Yoon.
\newblock Towards fast and accurate object detection in bio-inspired spiking
  neural networks through bayesian optimization.
\newblock {\em IEEE Access}, 9:2633--2643, 2020.

\bibitem{kim2020spiking}
Seijoon Kim, Seongsik Park, Byunggook Na, and Sungroh Yoon.
\newblock Spiking-yolo: spiking neural network for energy-efficient object
  detection.
\newblock In {\em Proceedings of the AAAI conference on artificial
  intelligence}, volume~34, pages 11270--11277, 2020.

\bibitem{li2022retinomorphic}
Jianing Li, Xiao Wang, Lin Zhu, Jia Li, Tiejun Huang, and Yonghong Tian.
\newblock Retinomorphic object detection in asynchronous visual streams.
\newblock In {\em Proceedings of the AAAI Conference on Artificial
  Intelligence}, volume~36, pages 1332--1340, 2022.

\bibitem{li2022spike}
Yang Li, Xiang He, Yiting Dong, Qingqun Kong, and Yi Zeng.
\newblock Spike calibration: Fast and accurate conversion of spiking neural
  network for object detection and segmentation.
\newblock {\em arXiv preprint arXiv:2207.02702}, 2022.

\bibitem{lien2022sparse}
Hong-Han Lien and Tian-Sheuan Chang.
\newblock Sparse compressed spiking neural network accelerator for object
  detection.
\newblock {\em IEEE Transactions on Circuits and Systems I: Regular Papers},
  69(5):2060--2069, 2022.

\bibitem{lin2014microsoft}
Tsung-Yi Lin, Michael Maire, Serge Belongie, James Hays, Pietro Perona, Deva
  Ramanan, Piotr Doll{\'a}r, and C~Lawrence Zitnick.
\newblock Microsoft coco: Common objects in context.
\newblock In {\em European conference on computer vision}, pages 740--755.
  Springer, 2014.

\bibitem{litjens2017survey}
Geert Litjens, Thijs Kooi, Babak~Ehteshami Bejnordi, Arnaud Arindra~Adiyoso
  Setio, Francesco Ciompi, Mohsen Ghafoorian, Jeroen~Awm Van Der~Laak, Bram
  Van~Ginneken, and Clara~I S{\'a}nchez.
\newblock A survey on deep learning in medical image analysis.
\newblock {\em Medical image analysis}, 42:60--88, 2017.

\bibitem{liu2019live}
Qian Liu, Ole Richter, Carsten Nielsen, Sadique Sheik, Giacomo Indiveri, and
  Ning Qiao.
\newblock Live demonstration: face recognition on an ultra-low power
  event-driven convolutional neural network asic.
\newblock In {\em Proceedings of the IEEE/CVF Conference on Computer Vision and
  Pattern Recognition Workshops}, pages 0--0, 2019.

\bibitem{liu2016ssd}
Wei Liu, Dragomir Anguelov, Dumitru Erhan, Christian Szegedy, Scott Reed,
  Cheng-Yang Fu, and Alexander~C Berg.
\newblock Ssd: Single shot multibox detector.
\newblock In {\em European conference on computer vision}, pages 21--37.
  Springer, 2016.

\bibitem{maass1997networks}
Wolfgang Maass.
\newblock Networks of spiking neurons: the third generation of neural network
  models.
\newblock {\em Neural networks}, 10(9):1659--1671, 1997.

\bibitem{merolla2014million}
Paul~A Merolla, John~V Arthur, Rodrigo Alvarez-Icaza, Andrew~S Cassidy, Jun
  Sawada, Filipp Akopyan, Bryan~L Jackson, Nabil Imam, Chen Guo, Yutaka
  Nakamura, et~al.
\newblock A million spiking-neuron integrated circuit with a scalable
  communication network and interface.
\newblock {\em Science}, 345(6197):668--673, 2014.

\bibitem{neftci2019surrogate}
Emre~O Neftci, Hesham Mostafa, and Friedemann Zenke.
\newblock Surrogate gradient learning in spiking neural networks: Bringing the
  power of gradient-based optimization to spiking neural networks.
\newblock {\em IEEE Signal Processing Magazine}, 36(6):51--63, 2019.

\bibitem{poole2016exponential}
Ben {Poole}, Subhaneil {Lahiri}, Maithra {Raghu}, Jascha {Sohl-Dickstein}, and
  Surya {Ganguli}.
\newblock {Exponential expressivity in deep neural networks through transient
  chaos}.
\newblock {\em arXiv e-prints}, page arXiv:1606.05340, June 2016.

\bibitem{poon2011neuromorphic}
Chi-Sang Poon and Kuan Zhou.
\newblock Neuromorphic silicon neurons and large-scale neural networks:
  challenges and opportunities.
\newblock {\em Frontiers in neuroscience}, 5:108, 2011.

\bibitem{posch2010qvga}
Christoph Posch, Daniel Matolin, and Rainer Wohlgenannt.
\newblock A qvga 143 db dynamic range frame-free pwm image sensor with lossless
  pixel-level video compression and time-domain cds.
\newblock {\em IEEE Journal of Solid-State Circuits}, 46(1):259--275, 2010.

\bibitem{redmon2016you}
Joseph Redmon, Santosh Divvala, Ross Girshick, and Ali Farhadi.
\newblock You only look once: Unified, real-time object detection.
\newblock In {\em Proceedings of the IEEE conference on computer vision and
  pattern recognition}, pages 779--788, 2016.

\bibitem{redmon2018yolov3}
Joseph Redmon and Ali Farhadi.
\newblock Yolov3: An incremental improvement.
\newblock {\em arXiv preprint arXiv:1804.02767}, 2018.

\bibitem{roy2019towards}
Kaushik Roy, Akhilesh Jaiswal, and Priyadarshini Panda.
\newblock Towards spike-based machine intelligence with neuromorphic computing.
\newblock {\em Nature}, 575(7784):607--617, 2019.

\bibitem{sengupta2019going}
Abhronil Sengupta, Yuting Ye, Robert Wang, Chiao Liu, and Kaushik Roy.
\newblock Going deeper in spiking neural networks: Vgg and residual
  architectures.
\newblock {\em Frontiers in neuroscience}, 13:95, 2019.

\bibitem{wu2018spatio}
Yujie Wu, Lei Deng, Guoqi Li, Jun Zhu, and Luping Shi.
\newblock Spatio-temporal backpropagation for training high-performance spiking
  neural networks.
\newblock {\em Frontiers in neuroscience}, 12:331, 2018.

\bibitem{wu2019direct}
Yujie Wu, Lei Deng, Guoqi Li, Jun Zhu, Yuan Xie, and Luping Shi.
\newblock Direct training for spiking neural networks: Faster, larger, better.
\newblock In {\em Proceedings of the AAAI conference on artificial
  intelligence}, volume~33, pages 1311--1318, 2019.

\bibitem{xiang2022spiking}
Shuiying Xiang, Tao Zhang, Shuqing Jiang, Yanan Han, Yahui Zhang, Chenyang Du,
  Xingxing Guo, Licun Yu, Yuechun Shi, and Yue Hao.
\newblock Spiking siamfc++: Deep spiking neural network for object tracking.
\newblock {\em arXiv preprint arXiv:2209.12010}, 2022.

\bibitem{10032591}
Man Yao, Guangshe Zhao, Hengyu Zhang, Yifan Hu, Lei Deng, Yonghong Tian, Bo Xu,
  and Guoqi Li.
\newblock Attention spiking neural networks.
\newblock {\em IEEE Transactions on Pattern Analysis and Machine Intelligence},
  45(8):9393--9410, 2023.

\bibitem{yao2023attention}
Man Yao, Guangshe Zhao, Hengyu Zhang, Yifan Hu, Lei Deng, Yonghong Tian, Bo Xu,
  and Guoqi Li.
\newblock Attention spiking neural networks.
\newblock {\em IEEE Transactions on Pattern Analysis and Machine Intelligence},
  2023.

\bibitem{zhang2019object}
Chen Zhang and Joohee Kim.
\newblock Object detection with location-aware deformable convolution and
  backward attention filtering.
\newblock In {\em Proceedings of the IEEE/CVF Conference on Computer Vision and
  Pattern Recognition}, pages 9452--9461, 2019.

\bibitem{zhang2022spiking}
Jiqing Zhang, Bo Dong, Haiwei Zhang, Jianchuan Ding, Felix Heide, Baocai Yin,
  and Xin Yang.
\newblock Spiking transformers for event-based single object tracking.
\newblock In {\em Proceedings of the IEEE/CVF Conference on Computer Vision and
  Pattern Recognition}, pages 8801--8810, 2022.

\bibitem{zheng2021going}
Hanle Zheng, Yujie Wu, Lei Deng, Yifan Hu, and Guoqi Li.
\newblock Going deeper with directly-trained larger spiking neural networks.
\newblock In {\em Proceedings of the AAAI Conference on Artificial
  Intelligence}, volume~35, pages 11062--11070, 2021.

\bibitem{zhu2022event}
Lin Zhu, Xiao Wang, Yi Chang, Jianing Li, Tiejun Huang, and Yonghong Tian.
\newblock Event-based video reconstruction via potential-assisted spiking
  neural network.
\newblock In {\em Proceedings of the IEEE/CVF Conference on Computer Vision and
  Pattern Recognition}, pages 3594--3604, 2022.

\bibitem{zhu2020deformable}
Xizhou Zhu, Weijie Su, Lewei Lu, Bin Li, Xiaogang Wang, and Jifeng Dai.
\newblock Deformable detr: Deformable transformers for end-to-end object
  detection.
\newblock {\em arXiv preprint arXiv:2010.04159}, 2020.

\end{thebibliography}
}

\appendix
\section{Proof of Gradient Norm Equality}
\begin{definition}
\textbf{(General Linear Transform) }Let $f(x)$ be a transform whose Jacobian matrix is $\boldsymbol{J}$. f is called general linear transform when it satisfies:

\begin{equation}
    E\left[\frac{\|\boldsymbol{f}(\boldsymbol{x})\|_{2}^{2}}{\operatorname{len}(\boldsymbol{f}(\boldsymbol{x}))}\right]=\phi\left(\boldsymbol{J J}^{T}\right)E\left[\frac{\|\boldsymbol{x}\|_{2}^{2}}{\operatorname{len}(\boldsymbol{x})}\right] .
\end{equation}
\end{definition}

\begin{lemma}[Multiplication](Theorem 4.1 in~\cite{chen2020tpami})
Given $\boldsymbol{J}:=\prod_{j=L}^1 \boldsymbol{J}_j$, where $\{\boldsymbol{J}_j\in \mathbb{R}^{m_j\times m_{j-1}}\}$ is a series of independent random matrices. If $(\prod_{j=L}^1\boldsymbol{J}_j)(\prod_{j=L}^1\boldsymbol{J}_j)^T$ is at least the $1^{st}$ moment unitarily invariant, we have 
\begin{equation}
    \phi\left((\prod_{j=L}^1\boldsymbol{J}_j)(\prod_{j=L}^1 \boldsymbol{J}_j)^T \right)=\prod_{j=L}^1\phi(\boldsymbol{J}_j {\boldsymbol{J}_j}^T).
\end{equation}
\label{lemma:multiplication}
\end{lemma}

\begin{lemma}[Addition] (Theorem 4.2 in \cite{chen2020tpami}) Given $\boldsymbol{J}:=\prod_{j=L}^1 \boldsymbol{J}_j$, where $\{\boldsymbol{J}_j\in \mathbb{R}^{m_j\times m_{j-1}}\}$ is a series of independent random matrices. If at most one matrix in $\boldsymbol{J}_j$ is not a central matrix, we have
\begin{equation}
    \phi(\boldsymbol{JJ}^T)=\sum_j\phi(\boldsymbol{J}_j {\boldsymbol{J}_j}^T).
\end{equation}
\label{lemma:addition}
\end{lemma}

\begin{proposition}
For EMS-Block1 and EMS-Block2, the Jacobian matrix of the block can be represented as $\phi(\boldsymbol{J _{j}J_{j}^{T}})=\frac{2}{\alpha_{2}^{j-1}}$.
\end{proposition}

\begin{proof}
\textbf{proof of EMS-Block1.} Since the EMS-Blocks have 2 paths, the residual path and the shortcut path while separately name the Jacobian matrix of two paths (of block) as $\boldsymbol{J}_{res}$ and $\boldsymbol{J}_{sc}$. $l$ is the layer number of the block, and it will be omitted where there is no ambiguity.

For the residual path with $2$ LCB blocks, according to General Linear Transform, we have
$$
\alpha_{2}^{l, res} = \phi(\boldsymbol{J}_{res}\boldsymbol{J}_{res}^{T})\alpha_{2}^{l-1},
$$
The shortcut path is similar
$$
\alpha_{2}^{l, sc} = \phi(\boldsymbol{J}_{sc}\boldsymbol{J}_{sc}^{T})\alpha_{2}^{l-1},
$$
Here, $\alpha_{2}^{l-1}$ is the $2^{th}$ moment of the input data from $(l-1)^{th}$ block.
Because the initialized BN layer have the output with variance $1$ and mean $0$, $\alpha_{2}^{l, res}=\alpha_{2}^{l, sc}=1$.Thus
$$
\phi(\boldsymbol{J}_{res}\boldsymbol{J}_{res}^{T})=\frac{1}{\alpha_{2}^{l-1}},
$$
$$
\phi(\boldsymbol{J}_{sc}\boldsymbol{J}_{sc}^{T})=\frac{1}{\alpha_{2}^{l-1}}.
$$
By addition principle
\begin{align*}
 \phi(\boldsymbol{J}_{EMS-Block1}\boldsymbol{J}_{EMS-Block1}^{T})\\=\phi(\boldsymbol{J}_{res}\boldsymbol{J}_{res}^{T})+\phi(\boldsymbol{J}_{sc}\boldsymbol{J}_{sc}^{T})\\
 =\frac{2}{\alpha_{2}^{l-1}}.   
\end{align*}
\end{proof}

\begin{proof}
\textbf{proof of EMS-Block2.} Comparing with EMS-Block1, the EMS-Block2 extra have a concatenation at the shortcut path.

\begin{table*}[t]
\centering
\renewcommand\arraystretch{1.5}{
\setlength{\tabcolsep}{8mm}{
\begin{tabular}{cccc}
\hline
Stage                        & ResNet-10                                                     & ResNet-18                                                     & ResNet-34                                                     \\ \hline
Conv1                        & \multicolumn{3}{c}{3×3, 32, stride 2}                                                                                                                                                              \\
Conv2\_x                     & $\biggl[\begin{array}{c}\text{3x3, 32}\\ \text{3x3, 64}\end{array}\biggr]*1$ & $\biggl[\begin{array}{c}\text{3x3, 32}\\ \text{3x3, 64}\end{array}\biggr]*2$ & $\biggl[\begin{array}{c}\text{3x3,32}\\ \text{3x3, 64}\end{array}\biggr]*3$ \\ 
\multicolumn{1}{l}{Conv3\_x} & $\biggl[\begin{array}{c}\text{3x3, 64}\\ \text{3x3, 128}\end{array}\biggr]*1$& $\biggl[\begin{array}{c}\text{3x3, 64}\\ \text{3x3, 128}\end{array}\biggr]*2$ & $\biggl[\begin{array}{c}\text{3x3, 64}\\ \text{3x3, 128}\end{array}\biggr]*4$ \\
\multicolumn{1}{l}{Conv4\_x} &$\biggl[\begin{array}{c}\text{3x3, 128}\\ \text{3x3, 256}\end{array}\biggr]*1$ &$\biggl[\begin{array}{c}\text{3x3, 128}\\ \text{3x3, 256}\end{array}\biggr]*2$ & $\biggl[\begin{array}{c}\text{3x3, 128}\\ \text{3x3, 256}\end{array}\biggr]*6$ \\
Conv5\_x                     & $\biggl[\begin{array}{c}\text{3x3, 256}\\ \text{3x3, 512}\end{array}\biggr]*1$ & $\biggl[\begin{array}{c}\text{3x3, 256}\\ \text{3x3, 512}\end{array}\biggr]*2$ &$\biggl[\begin{array}{c}\text{3x3, 256}\\ \text{3x3, 512}\end{array}\biggr]*3$\\ \hline

\end{tabular}
}
}
\vspace{1mm}
\caption{\textbf{Model structures for ablation experiments.} x represents the current module repeated x times and the first module transformed in a reduced dimension. Compared with the original ResNet structure, the number of channels are resized here, while the final FC layer is removed.}
\vspace{-4mm}
\label{tab:resnet}
\end{table*}

According to the discussion about concatenation in~\cite{chen2020tpami}, we have
$$
\phi\left(\boldsymbol{J}_{j} \boldsymbol{J}_{j}^{T}\right)=\frac{c_{j-1}}{c_{j}}+\frac{\delta_{j}}{c_{j}} \phi\left(\boldsymbol{H}_{{j}} \boldsymbol{H}_{{j}}^{T}\right),
$$
Here $\boldsymbol{J}_{j}$ denote Jacobian matrix of the block of shortcut path without maxpooling layer. ${H}_{{j}}$ denote the Jacobian matrix of the LCB block. $c_{j-1}$ and $c_{j}$ denoted as the channel numbers for input and output of concatenation. And $\delta_{j}=c_{j}-c_{j-1}$. It is trivial that by adding the maxpooling layer and using general linear transform, shortcut path can be expressed as

\begin{align*}
\phi(\boldsymbol{J}_{sc}\boldsymbol{J}_{sc}^{T})=\frac{\alpha_{2}^{maxpool}}{\alpha_{2}^{l-1}}(\frac{c_{j-1}}{c_{j}}+\frac{\delta_{j}}{c_{j}} \phi\left(\boldsymbol{H}_{{j}} \boldsymbol{H}_{{j}}^{T}\right))\\
=\frac{1}{\alpha_{2}^{l-1}}(\frac{\alpha_{2}^{maxpool}c_{j-1}}{c_{j}}+\frac{\alpha_{2}^{maxpool}\delta_{j}}{c_{j}} \left(\frac{\alpha_{2}^{bn}}{\alpha_{2}^{maxpool}}\right)),
\end{align*}
Since the $2^{th}$ moment $\alpha_{2}^{l-1}$ is strictly controlled by the BN layers of former block, the $\alpha_{2}^{maxpool}$ is fixed too. Thus, let $\alpha_{2}^{bn}=\frac{2c_{j}-\alpha_{2}^{maxpool}c_{j-1}}{\delta_{j}}$ by proper initializing of BN layers, $\phi(\boldsymbol{J}_{sc}\boldsymbol{J}_{sc}^{T})=\frac{1}{\alpha_{2}^{l-1}}$ holds.

The other part of EMS-Block2 is similar with EMS-Block1, thus 
$$
\phi(\boldsymbol{J}_{EMSblock2}\boldsymbol{J}_{EMSblock2}^{T})=\frac{2}{\alpha_{2}^{l-1}}.
$$

\end{proof}

\begin{proposition}

For the EMS-ResNet, $\phi(\boldsymbol{JJ^{T}})\approx 1$ can be satisfied by control the $2^{th}$ moment of the input.
\end{proposition}
\begin{proof}
MS-Block is a typical resblock that have already been discussed in \cite{chen2020tpami}. Using general linear transform and addition principle, we have
$$
\alpha_{2}^{l-1}\phi(\boldsymbol{J}\boldsymbol{J}^{T})=\alpha_{2}^{l}=\alpha_{2}^{l-1}+1.
$$
And $\alpha_{2}^{l-1}$ is comes from the EMS-Block1 or EMS-Block2, where $\alpha_{2}^{l-1}$ is fixed at $2$. Thus, $\phi(\boldsymbol{J}_{MS-Block}\boldsymbol{J}_{MS-Block}^{T})=\frac{3}{2}$.

By using multiplication principle, the whole blocks have the property
$$
\phi(\boldsymbol{J}\boldsymbol{J}^{T})=\frac{3}{\alpha_{2}^{0}},
$$
where $\alpha_{2}^{0}$ is the $2^{th}$ moment of the output of BN in encoding layer.

After initialized the BN in encoding layer, $\alpha_{2}^{0}$ can be controlled to $3$ and then $\phi(\boldsymbol{J}\boldsymbol{J}^{T})\approx1$ holds. 
\end{proof}

\section{Datasets Introduction}

\paragraph{COCO2017 Dataset}
 COCO2017 Dataset~\cite{lin2014microsoft} is a large-scale object detection, segmentation, key-point detection, and captioning dataset. For object detection, its training set and test set contain 118K and 5K images, respectively. The instances of 80 categories are labeled with their classes and bounding boxes respectively.
 \paragraph{GEN1 Automotive Detection Dataset}
Event cameras possess outstanding properties compared with the traditional frame cameras. They have high dynamic range to overcome motion blur. Furthermore, objects are captured well even in low-light or overexposed scenes. The event $e_n$ (defined in \textbf{Sec 4.1}) in the event camera represents the change in light intensity $I$ of the pixel $(x_n,y_n)$, which can be formulated as:
\begin{equation}
    \ln{I(x_n,y_n,t_n)}-\ln{I(x_n,y_n,t_n-\Delta{t_n})}=p_n\theta_{th}
\end{equation}
where $\Delta{t_n}$ represents the temporal sampling interval. 

As the largest event camera-based dataset currently available,  Gen1 dataset~\cite{de2020large} contains two categories (pedestrians and cars), and 39 hours of automotive recordings in diverse scenarios. Gen1 is labeled manually by the gray level estimation feature of the ATIS sensor~\cite{posch2010qvga} with a resolution of 304$\times$204 pixels and more than 255,000 bounding box annotations are yielded in total. 

\begin{figure*}[!t]
\centering
\includegraphics[scale=0.5]{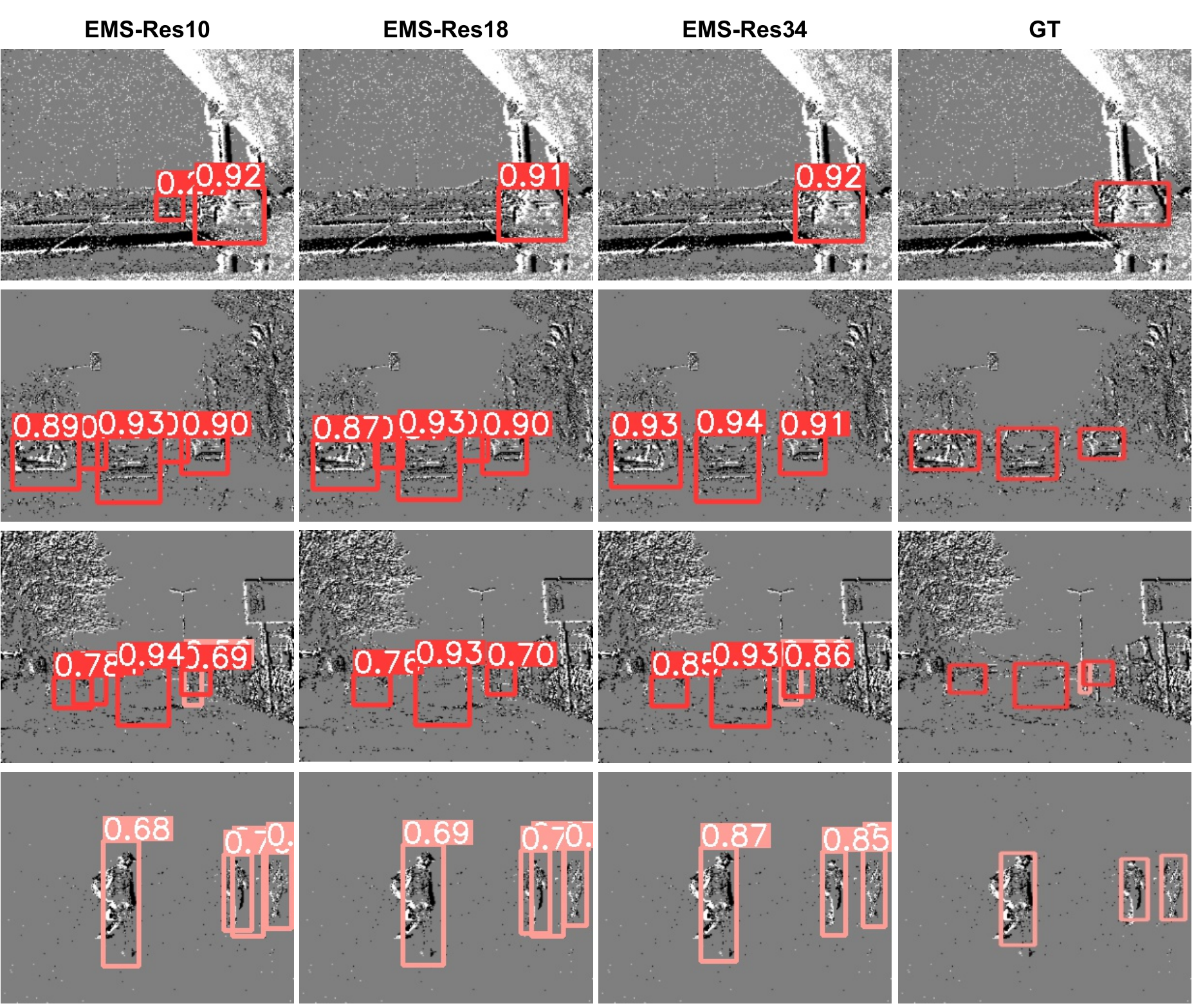}
\caption{{\textbf{More detailed comparison figures on the gen1 dataset}.}}
\label{fig:g1}
\end{figure*}

\section{More Detailed Experiments}
Here we provide more details on the experiments in \textbf{Sec 5.3} using the COCO2017 dataset. They are all based on the channel number reduction models (see Table~\ref{tab:resnet}), because we aim to validate the reliability of our experimental conclusions, not to achieve the optimal performance. In addition, we explore the impact of the last layer of LIF on the model performance. We use 4 Nvidia A100 GPUs and the SGD optimizer with a learning rate of 1E-2 for training.

\begin{table}[h]
\centering
\resizebox{240pt}{35pt}{
\begin{tabular}{@{}ccclcc@{}}
\toprule
Model     & \begin{tabular}[c]{@{}c@{}}mAP\\ @0.5\end{tabular} & \begin{tabular}[c]{@{}c@{}}mAP\\ @0.5:0.95\end{tabular} & \multicolumn{1}{c}{Params} & \begin{tabular}[c]{@{}c@{}}Firing\\ Rate\end{tabular} & \begin{tabular}[c]{@{}c@{}}Energy\\ Efficiency\end{tabular} \\ \midrule
Sew-Res18 & 0.345                                                & 0.183                                                     & 9.743M                     & 24.20\%                                                 & $3.31\times$                                                        \\
MS-Res18  & \multicolumn{1}{l}{0.345}                          & 0.184                                                   & 9.678M                     & 32.32\%                                                 & $3.55\times$                                                       \\
EMS-Res18 & \textbf{0.362}                                                & \textbf{0.201}                                                     & \textbf{9.523M  }                   & 38.75\%                                                 & \textbf{5.98}$\times$                                                      \\ \bottomrule
\end{tabular}
}
 \vspace{1mm}
\caption{Ablation studies of different residual blocks on COCO2017 dataset.}
\vspace{-4mm}
\label{tab:diffres_coco}
\end{table}
\paragraph{Different Residual Blocks}
We conduct additional experiments on the COCO2017 dataset to fully illustrate the effectiveness of EMS-ResNet. We train all the models for only 120 epochs with a batchsize of 64, and the time steps are set to 3. As shown in Table~\ref{tab:diffres_coco} , our model presents optimal performance with a high spiking rate, which may imply an increase in spiking rate, enabling better model feature extraction. At the same time, our model is fully spiked, and even with a high spike rate, it is still more energy-efficient than other models. We set the energy consumption of the ANN with the same structure as the baseline, denoted as $1\times$, and our full spike EMS-ResNet reduces the energy consumption up to 5.98 times.

\paragraph{Numbers of Residual Blocks}
We explore the effect of network depth on performance on the COCO2017 dataset. Here we set the time step to 1 and train for only 50 epochs. As shown in Table~\ref{diffnum_coco}, the network converges faster and recognizes objects more accurately as the depth increases.

\begin{table}[h]
\centering
\begin{tabular}{@{}ccccc@{}}
\toprule
Model     & \begin{tabular}[c]{@{}c@{}}mAP\\ @0.5\end{tabular} & \begin{tabular}[c]{@{}c@{}}mAP\\ @0.5:0.95\end{tabular} & Params & \begin{tabular}[c]{@{}c@{}}Firing\\ Rate\end{tabular} \\ \midrule
EMS-Res10 & 0.203                                              & 0.091                                                   & 6.387M  & 30.01\%                                                   \\
EMS-Res18 & 0.268                                              & 0.132                                                   & 9.523M & 28.56\%                                                   \\
EMS-Res34 & 0.335                                              & 0.178                                                   & 14.58M & 29.55\%                                                   \\ \bottomrule
\end{tabular}
\vspace{2mm}
\caption{Impact of different number of residual blocks on COCO2017 dataset.}
\vspace{-4mm}
\label{diffnum_coco}
\end{table}

\paragraph{Spiking Detection layer}
In the object detection task, it is necessary to consider how to convert the features of the spike trains into continuous value representations of the bounding box coordinates. This can be achieved by using either a non-spiking detection layer that directly feeds the last neuronal membrane potential or a spiking detection layer that uses rate-coding before different detection layers. From the experimental results (see Table~\ref{tab:diffdetction}), these two conversion methods have little effect on the performance of the model.

\begin{table}[h]
\centering
\begin{tabular}{@{}cccccc@{}}
\toprule
Dataset               & Model     & T & Params & \begin{tabular}[c]{@{}c@{}}mAP\\ @0.5\end{tabular} & \begin{tabular}[c]{@{}c@{}}mAP\\ @0.5:0.95\end{tabular} \\ \midrule
\multirow{2}{*}{COCO} &  Non-spiking & 3 & 9.523M & 0.318                                              & 0.165                                                   \\
                      & Spiking     & 3 & 9.523M & 0.305                                              & 0.157                                                   \\
\multirow{2}{*}{Gen1} & Non-spiking & 5 & 9.343M & 0.566                                              & 0.286                                                   \\
                      & Spiking     & 5 & 9.343M & 0.565                                              & 0.286                                                   \\ \bottomrule
\end{tabular}
\vspace{1mm}
\caption{Impact of spiking/non-spiking detection layer on the model performance.}
\vspace{-4mm}
\label{tab:diffdetction}

\end{table}

\end{document}